\documentclass[letterpaper, 10 pt, conference]{ieeeconf}  

\IEEEoverridecommandlockouts                              

\usepackage{graphicx} 

\usepackage{amsthm}
\usepackage{amsmath} 
\usepackage{amssymb}  
\usepackage{xcolor}

\usepackage{verbatim}
\usepackage{algorithm}
\usepackage{algorithmicx}
\usepackage[noend]{algpseudocode}
\usepackage{cite}

\usepackage{graphicx}
\usepackage{textcomp}
\usepackage{xcolor}
\usepackage{booktabs}

\usepackage{multirow,array}

\usepackage{subcaption}
\usepackage{graphicx}
\usepackage{hyperref}

\usepackage[utf8]{inputenc}
\usepackage[english]{babel}

\newtheorem{theorem}{Theorem}

\newtheorem{lemma} {Lemma}
\newtheorem{property} {Property}

\theoremstyle{definition}
\newtheorem{definition}{Definition}

\renewcommand{\algorithmicrequire}{\textbf{Input: }}
\renewcommand{\algorithmicensure}{\textbf{Output: }}

\begin{document}

\title{Contact Mode Guided Sampling-Based Planning \\for Quasistatic Dexterous Manipulation in 2D
\thanks{This work was supported by the NSF Grant IIS-1909021.}
}
\author{Xianyi Cheng, Eric Huang, Yifan Hou, and Matthew T. Mason \\ Carnegie Mellon University \\
\{xianyic,erich1,yifanh\}@andrew.cmu.edu, mattmason@cmu.edu }

\maketitle

\begin{abstract}
The discontinuities and multi-modality introduced by contacts make manipulation planning challenging. Many previous works avoid this problem by pre-designing a set of high-level motion primitives like grasping and pushing. However, such motion primitives are often not adequate to describe dexterous manipulation motions. In this work, we propose a method for dexterous manipulation planning at a more primitive level. The key idea is to use contact modes to guide the search in a sampling-based planning framework. Our method can automatically generate contact transitions and motion trajectories under the quasistatic assumption. In the experiments, this method sometimes generates motions that are often pre-designed as motion primitives, as well as dexterous motions that are more task-specific%
\footnote{The code is available at \url{https://github.com/XianyiCheng/Dexterous-Manipulation-Planning-2D}. The supplementary video is available at \url{https://youtu.be/2yYYLN3JAbs}. }.
\end{abstract}

\section{Introduction}
\label{sec:intro}

It is hard to estimate the set of manipulation skills used by humans or animals, and even harder to reproduce all of those skills. 
Byrne \cite{byrne2001manual} documented 72 functionally distinct manipulation primitives used by foraging mountain gorillas. 
Nakamura \cite{nakamura2017complexities} studied human grasping behaviors in grocery stores and noticed that existing taxonomies cannot categorize all the observed behaviors. Within robotics, many researchers have explored individual skills at depth, including pushing \cite{mason1986mechanics}\cite{lynch1996stable}, pivoting \cite{aiyama1993pivot,holladay2015general,hou2018fast}, tumbling \cite{maeda2004motion}, whole-body manipulation \cite{salisbury1988whole}, extrinsic dexterity \cite{ChavanDafle2014extrinsic}, grasping \cite{eppner2015exploitation}\cite{eppner2015planning}, shared grasping \cite{hou2020manipulation}, etc. Clearly, designing a high-quality robotic manipulation skill for a specific task takes significant human labor. 
When planning for robotic manipulation tasks, it is therefore desirable to find an approach that reduces the programming effort required and creates skills that generalize to many scenarios.

Manipulation tasks can be broken down into discrete and continuous parts. The discrete part is about the changes of contacts and their modes. 
For challenging manipulation problems, the planner will need to deal with the combinatorial explosion associated with making a sequence of discrete choices. Even if the sequence is known, the next challenge is to find a piecewise continuous trajectory which satisfies dynamic equations, force limits, and endpoint constraints for each piece of the sequence. 

In this paper, we focus on quasistatic dexterous manipulation planning for robot/finger contact point motions in 2D (an example is shown in Figure \ref{fig:intro}), without pre-defined high-level motion primitives. 
To this end, we propose a sampling-based planner guided by contact modes. The guidance of contact modes is similar to the automatic generation of ``motion primitives''. This planner can generate solutions that cannot be found by a regular sampling-based planner, and are hard to describe as sequences of manually designed macro motion primitives.
The core idea is to find all possible discrete modes through contact mode enumeration; utilizing contact mode constraints, we project randomly sampled configurations onto manifolds defined by contact modes and forward integrate towards the goal configuration on the lower-dimensional manifolds to allow efficient exploration over these continuous spaces. The contact mode guidance enables the discovery of solutions in the zero-volume manifolds in the C-space, which have zero possibility of being randomly sampled. The sampling-based nature of this method enables fast exploration over large continuous and discrete spaces to plan contact-rich motions. A random tree version of this planner is proved to be probabilistic complete in Section~\ref{sec:pc}. 
As a result, in Section~\ref{sec:experiments}, 
this planner shows the advantages of simplicity and speed --- it requires little hand-tuning for different tasks in 2D and can generate complicated trajectories within seconds. 
This framework can also integrate a stability margin method \cite{hou2020manipulation} to pick robust motions for robot execution in Section \ref{sec:robo_experiments}.
\begin{figure}
    \centering
    \includegraphics[width=\columnwidth]{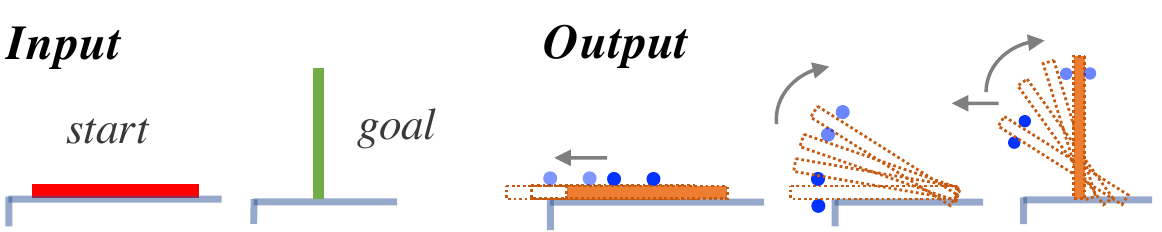}
    \caption{Pickup a blade: one example of the 2D manipulation planning problems to solve in this paper. }
    \vspace{-0.5cm}
    \label{fig:intro}
\end{figure}

\section{Related Work}
\label{sec:literature}

\subsection{Dexterous Manipulation Planning}

To reduce the complexity of planning through contacts, many works use predefined high-level motion primitives \cite{toussaint2018differentiable,lozano2014sequential,barry2013hierarchical,woodruff2017planning,berenson2011constrained}. Often it becomes a Task and Motion Planning problem \cite{lozano2014sequential}. Efficient search and optimization algorithms have been developed to solve these motion sequencing/planning problems\cite{lozano2014sequential}\cite{toussaint2018differentiable}. Hierarchical planning\cite{lee2015hierarchical}\cite{barry2013hierarchical} is designed to reduce the search space by dividing the whole planning problem into sub-problems. 
Sampling-based planning methods like CBiRRT are developed to efficiently explore the manifolds by a variety of constraints \cite{simeon2004manipulation}\cite{berenson2011constrained}. 
Most of the methods above require predefined states or primitives, which require extra engineering efforts given a new task or environment. 
Their solutions are also confined to be the combinatorics of predefined states/primitives.

Contact formations \cite{xiao2001automatic} have been explored in \cite{trinkle1991dexterous}\cite{ji2001planning}\cite{tang2008automatic}\cite{lee2015hierarchical} to plan motions between two rigid bodies. 
This paper shares a similar idea with them: use contacts to decompose the search space into smaller chunks. Search and planning within contact formations are later combined into a complete solution. 

If we consider each point contact as a robot instead of a finger tip, this work is also related to multi-robot/distributed manipulation \cite{bohringer1997distributed}. The applications are mostly focused on planar manipulation. Multi-robot box-pushing\cite{mataric1995cooperative}, furniture moving\cite{rus1995moving}, part reorienting\cite{bohringer1997distributed}, caging with obstacles\cite{fink2008multi}, etc. have been explored in this area. 

Our method can be seen as an extension of constrained sampling-based planning algorithm \cite{berenson2011constrained}\cite{kingston2019exploring}. Our method uses instantaneously enumerated contact modes to find and explore various lower-dimensional manifolds of the C-space, while a regular constrained sampling-based planner only exploits known constrained manifolds provided by the users. One can view our contact mode guidance as the automatic generation of ``motion primitives''.

\subsection{Contact-rich Trajectory Optimization}
Contact-Invariant Optimization methods\cite{mordatch2012contact}\cite{mordatch2012discovery} produce complex whole-body and manipulation behaviors in simulation, assuming soft contacts which may violate physics laws. Dynamic manipulation planning for rigid bodies are explored in \cite{posa2014direct} \cite{Sleiman2019CIO} \cite{doshi2020icra} \cite{aceituno-cabezas2020rss}. They can plan simple manipulation actions with small numbers of contact transitions, such as pushing and pivoting. 
Trajectory optimization could be time-consuming and intractable without good initialization, especially when contact schedules are not specified \cite{aceituno-cabezas2020rss,posa2014direct,manchester2020variational}.
To our best knowledge, there has not been any trajectory optimization method that can solve the tasks of a similar level of complexity presented in this paper. 

\section{Problem Description}
\label{sec:problemformulation}


The inputs to our method are the start and goal configurations of the object, the geometries, and the properties of the object and the environment:
\begin{enumerate}
    \item \textbf{Object start configuration}: 
    $q_{\mathrm{start}} \in SE(2)$.
    
    \item \textbf{Object goal}: object goal configuration $q_{\mathrm{goal}} \in SE(2)$ and the allowable goal region $Q_{\mathrm{goal}} \subset SE(2)$.
    
    \item \textbf{Object properties}: a rigid body $\mathcal{O}$ with known polygonal geometry (can be non-convex), and the friction coefficients between the object with environment $\mu_{\mathrm{env}}$ and the object with the manipulator $\mu_{\mathrm{mnp}}$.
    
    \item \textbf{Environment}: an environment $\mathcal{E}$ with known geometries. As it is 2D, the environment could either be in the horizontal plane, like a tabletop with friction, or in the vertical plane with gravity.
    
    \item \textbf{Manipulator}: we assume a simplified manipulator model which is just $N_{\mathrm{mnp}}$ point contacts. The manipulator's configuration 
    can be represented by the contact locations on the object $[p^{\mathrm{mnp}}_1, p^{\mathrm{mnp}}_2, \dots, p^{\mathrm{mnp}}_{N_{\mathrm{mnp}}}]$.
\end{enumerate}

Our method outputs a trajectory $\pi$ that moves the object from its start to goal. The trajectory is a sequence of object motions, contacts, and contact modes. 
The trajectory $\pi$ includes: 
\begin{enumerate}
    \item \textbf{Object motion}: object configuration $q(t)$ at step $t$. 
    \item \textbf{Environment contacts}: the set of contact points of the object with the environment. The number of environment contacts $N_{\mathrm{env}}(t)$ at each step may vary. The $k$th environment contact is specified by its contact location $p^{\mathrm{env}}_k(t)$, contact normal $n^{\mathrm{env}}_k(t)$ and contact force $\lambda^{\mathrm{env}}_k(t)$. In practice, they are generated by the collision detection of the object and the environment.
    
    \item \textbf{Manipulator contacts}: the set of contact points of the object with the manipulator: $c^{\mathrm{mnp}}(t) = [c^{\mathrm{mnp}}_1(t), c^{\mathrm{mnp}}_2(t), \dots, c^{\mathrm{mnp}}_{N_{\mathrm{mnp}}}(t)]$. 
    The $k$th manipulator contact is specified by its contact location $p^{\mathrm{mnp}}_k(t)$, contact normal $n^{\mathrm{mnp}}_k(t)$ and contact force $\lambda^{\mathrm{mnp}}_k(t)$. 

    \item \textbf{Contact mode}: the contact mode $m(t)$ of the environment contacts and the manipulator contacts: $m(t) \in \{\textit{separate}, \textit{ fixed}, \textit{ right-slide}, \textit{ left-slide}\}^{N_{\mathrm{mnp}} + N_{\mathrm{env}}(t)}$. 
\end{enumerate}


We make the following model assumptions:
(1) Rigid body: the object, environment, and manipulator are rigid.
(2) Quasistatic: the speeds of the object motions are low. Inertial forces and impacts are not considered.
(3) The manipulator contact modes are \textit{fixed} only.
\section{quasistatic multi-contact manipulation model}
\label{sec:mechanics}
In this section, we introduce the contact mechanics and derive the motion generation method based on contact modes that are later used in our planning framework in Section \ref{sec:rrt}.

\subsection{Contact Mode Constraints}

Contact constraints are difficult to deal with due to the presence of multi-modality and discontinuities. The contact constraints can be modeled as complementarity constraints, and they are the main source of problems during contact implicit trajectory optimization (CITO). 
Assuming Coulomb's law, we break down the contact constraints into sets of linear equality and inequality constraints by enumerating all feasible contact modes of the system. Each feasible contact mode maps to a set of local dynamical equations for the system. 
In 2D, the set of kinematically feasible contact modes of $n$ contacts can be defined as $\mathcal{M} \subseteq \{ \textit{fixed}, \textit{right-slide}, \textit{left-slide}, \textit{separate}\}^n$ by the relative motion of the object in the contact frames. All feasible 2D contact modes can be enumerated in $O(n^2)$ \cite{Mason}.
%

By specifying the contact mode $m \in \mathcal{M}$ for a quasistatic system, we obtain linear constraints on object motions and contact forces.

\noindent \textbf{Velocity constraints:} 
Contact velocities can be written as:
\begin{equation}
v_c = G^T v_o - \begin{bmatrix}J \dot{q} \\ 0\end{bmatrix}
\end{equation}
where $v_o$ is the object velocity; $G$ is the grasp map \cite{murray1994mathematical}; $\dot{q}$ is the manipulator joint velocity and $J$ is the manipulator's Jacobians; the $0$ part is for environment contacts that have zero velocities. For point manipulator, $\dot{q}$ can be simply written as the point contact's translational velocity $[\dot{x_1}, \dot{y_1}, \dots, \dot{x_n}, \dot{y_n}]$ and $J$ is an identity matrix.

For the $i$th contact, its contact mode $m_i$ indicates that its relative contact velocity $v_c^i$ is constrained as:
\begin{equation}
    \label{eqn:motion}
    \begin{cases}
    v_{c,n}^i > 0  \quad \text{if} \quad m_i = \textit{separate}\\
    v_{c,n}^i = 0, v_{c,t}^i = 0  \quad \text{if} \quad m_i = \textit{fixed}\\
    v_{c,n}^i = 0, v_{c,t}^i > 0  \quad \text{if} \quad m_i = \textit{right-slide}\\
    v_{c,n}^i = 0, v_{c,t}^i < 0  \quad \text{if} \quad m_i = \textit{left-slide}
    \end{cases}
\end{equation}
where $v_{c,n}^i$ and $v_{c,t}^i$ are $v_c^i$ in contact normal and tangential directions respectively.

\noindent \textbf{Force constraints:}
Under Coulomb's friction law, the force constraints under a mode $m_i$ can be written as:
\begin{equation}
    \label{eqn:force}
    \begin{cases}
    \lambda_n^i = 0, {\lambda_t^i}^{+}=0, {\lambda_t^i}^{-} = 0 \quad \text{if} \quad m_i = \textit{separate} \\
    \mu \lambda_n^i - {\lambda_t^i}^{+}  > 0, \mu \lambda_n^i - {\lambda_t^i}^{-}  > 0 \quad \text{if} \quad m_i = \textit{fixed} \\
    \mu \lambda_n^i = {\lambda_t^i}^{-}, {\lambda_t^i}^{+} = 0 \quad \text{if} \quad m_i = \textit{right-slide} \\
    \mu \lambda_n^i = {\lambda_t^i}^{+}, {\lambda_t^i}^{-} = 0 \quad \text{if} \quad m_i = \textit{left-slide}
    \end{cases}
\end{equation}
where $\lambda_n^i$ is the contact normal force, ${\lambda_t^i}^{+}$ and ${\lambda_t^i}^{-}$ are the contact tangential force magnitudes in the positive direction (right-slide) and negative direction (left-slide) respectively.

Given a contact mode $m = [m_1,m_2,\dots,m_N]$ for all the contacts on the object, let all contact forces be $\lambda = [ \lambda_n^1, {\lambda_t^1}^+, {\lambda_t^1}^-, \dots, \lambda_n^N, {\lambda_t^N}^+, {\lambda_t^N}^- ] $, combining Equation \ref{eqn:motion} and \ref{eqn:force} for all the contacts, the velocity and force constraints can be written as linear equalities and inequalities:
\begin{equation}
    \label{eqn:mode}
    \begin{split}
    A_{\mathrm{ineq}} \begin{bmatrix}v & \lambda \end{bmatrix}^T > b_{\mathrm{ineq}}\\
    A_{\mathrm{eq}} \begin{bmatrix}v & \lambda \end{bmatrix}^T = b_{\mathrm{eq}}
    \end{split}
\end{equation}

Given the desired contact mode and desired object velocity $v_d$, we can find the closest feasible object velocity by solving a quadratic programming problem: 
\begin{equation}
\label{eqn:optvel}
    \min\limits_{v_o, \dot{q}, \lambda,} \| v_d - v_o \| + \epsilon \lambda^T \lambda 
\end{equation}
where $\epsilon \lambda^T \lambda$ is a regularization term on the contact forces.
This optimization problem is subject to the contact mode constraints in Equation \ref{eqn:mode} and the static equilibrium constraints below:
\begin{equation}
   G\lambda + F_{\mathrm{external}} = 0
\end{equation}
where $v_o$ is the object velocity; $G$ and $\lambda$ are the grasp map and contact forces for all contacts; $F_{external}$ include forces and torques that are not modeled in $\lambda$, such as gravity.


\subsection{Projected Forward Integration}
\label{sec:forward}

\begin{figure}
    \small
    \centering
    \begin{subfigure}{\columnwidth}
    \includegraphics[width=\columnwidth]{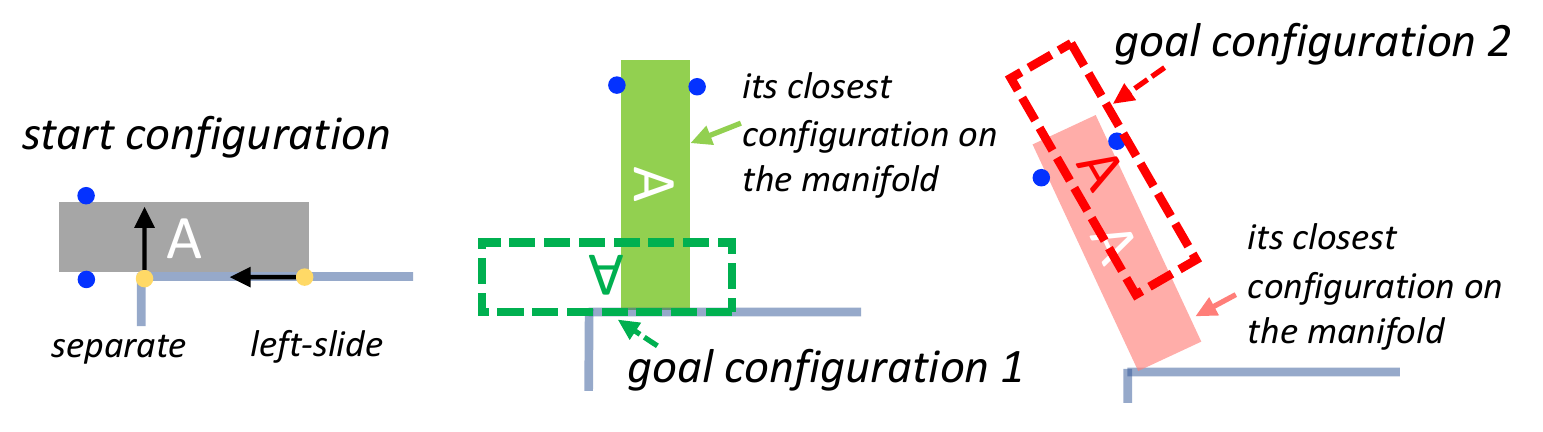}
    \caption{The gray solid block shows the start pose of the object with two blue manipulator contacts (both \textit{fixed}) and two yellow environment contacts (contact mode: {\textit{separate}, \textit{left slide}}). In goal configuration 1, the block is 180 degree flipped (the dotted green block). Its closest configuration on the contact mode manifold is the solid light green block, as current mode is interrupted by new contact with the table. For goal configuration 2 (the dotted red block), its closest configuration is shown as the solid light red block.} \label{fig:1a}
    \end{subfigure}
    \begin{subfigure}{\columnwidth}
    \includegraphics[width=\columnwidth]{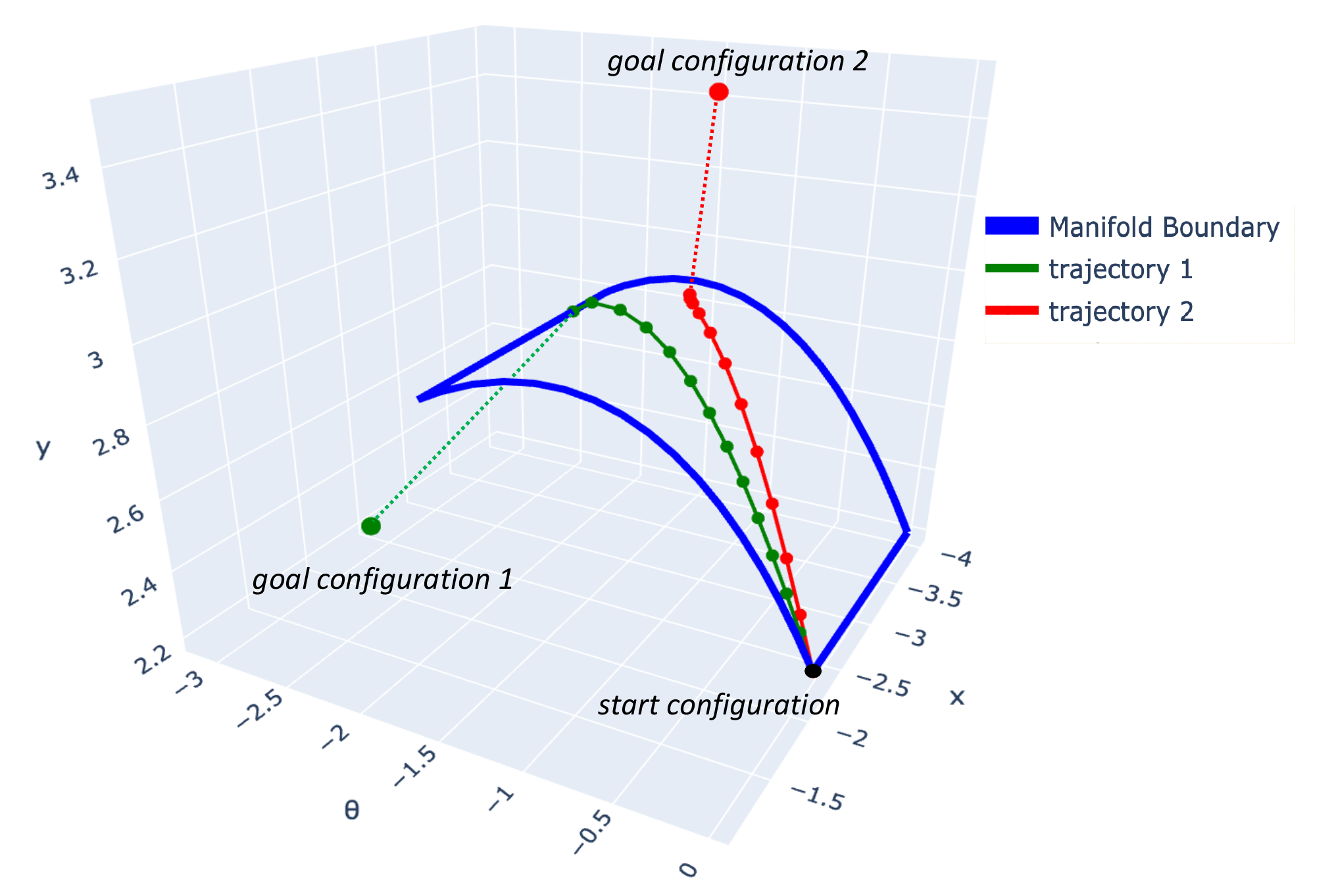}
    \caption{The iterative process of forward integration. At every timestep it moves towards the goal configuration on the manifold. For goal configuration 1, this process finds a point on the boundary. For goal configuration 2, it finds a point in the interior.} \label{fig:2a}
    \end{subfigure}
    \caption{Two examples of the forward integration process for a specified contact mode in Section \ref{sec:forward}. 
    }
    \label{fig:forward}
    \vspace{-0.3cm}
\end{figure}



Starting from an object configuration, all reachable configurations under a contact mode form a manifold with boundary in the object configuration space. However, this manifold does not has explicit representation since the contact mode only provides linear velocity and force constraints for each point on the manifold, thus we cannot directly 
project a configuration onto this manifold. In this section, we describe an iterative method that finds the closest reachable configuration to a goal configuration on a contact mode manifold, as visualized in Figure \ref{fig:forward}. 

Given a desired configuration $q_d$, forward integration moves towards $q_d$ on the contact mode manifold by iteratively integrate $v_o$ computed in Equation \ref{eqn:optvel} through the Euler method:
\begin{equation}
    q_{k+1} = Tr(q_k, h {v_o}^k)
\end{equation}
where $h$ is the length of the time-steps, $Tr$ is the rigid body transformation computed from the body velocity $h {v_o}^k$, applied on $q_k$ \cite{murray1994mathematical}. At each timestep, the desired velocity $v_d$ is computed as the body velocity between $q_k$ and $q_{\mathrm{rand}}$ in the twist coordinate \cite{murray1994mathematical}, and environment contacts are updated by collision checking. 
The forward integration stops when:
(1) ${v_o}^k$ is zero: $q_k$ is the closest configuration on the contact manifold to $q_d$
(2) No feasible ${v_o}^k$: the static equilibrium cannot be maintained anymore
(3) Collision: the fingers collide with the environment, or new object-environment contacts are made. In the simulation, we use linear interpolation to go back to zero contact distances.


The following property states that this forward integration process is a projection operation. 
\begin{property}
\label{lemma:project}
The forward integration process described in Section \ref{sec:forward} is a projection operator $P(\cdot)$ to a manifold $\mathcal{M}$ which has the following properties:
\begin{enumerate}
    \item $P(q) = q$ if and only if $q \in \mathcal{M}$
    \item If $q \in \mathcal{M}$ is the closest point to $q_1 \notin \mathcal{M}$, $P(q_1) = q$.
\end{enumerate}
\end{property}

\section{Contact mode guided sampling based planning}
\label{sec:rrt}

While the exploit of contact modes can be used in different planning frameworks, we adopt a rapidly exploring random tree (RRT) as the high-level planning framework (Algorithm \ref{alg:rrt}). 
During the \textsc{extend} process of RRT, our method project the randomly sampled object configuration onto the contact mode manifolds of its nearest node. This ensures random sampling on zero-volume manifolds in the C-space. As the proofs show in Section \ref{sec:pc}, a random tree version of our method (Algorithm \ref{alg:rrt-complete}) is probabilistic complete. 


\subsection{Planning Framework}
\label{sec:main-alg}
\begin{figure}
\vspace{-1.5mm}
\end{figure}
\begin{algorithm}[t]
    \small
    \caption{Contact Mode Based Manipulation Planner} \label{alg:rrt}
    \algorithmicrequire $q_{\mathrm{start}}$, $q_{\mathrm{goal}}$ \newline
    \algorithmicensure tree $\mathcal{T}$
    \begin{algorithmic}[1]
        \State $\mathcal{T}\text{.add-node}(q_{\mathrm{start}})$ 
        \While {(\textit{Time limit has not been reached})}
        \State $q_{\mathrm{rand}} \gets \Call{sample-object-config}{q_{\mathrm{goal}}}$
         
        \State $q_{\mathrm{near}}\gets \Call{nearest-neighbor}{\mathcal{T},q_{\mathrm{rand}}}$
        \State $c^{\mathrm{mnp}}_{\mathrm{near}} \gets \Call{node-property}{q_{\mathrm{near}}}$
        \State $c^{\mathrm{env}}_{\mathrm{near}} \gets \Call{collision-detection}{q_{\mathrm{near}}}$
        \State $\mathcal{M} \gets \Call{contact-mode-enumeration}{q_{\mathrm{near}}}$
        \For {$m \in \mathcal{M}$} 
        \Comment{Iterate every contact mode}
        \State $\Call{Extend}{m, q_{\mathrm{near}}, q_{\mathrm{rand}}, c^{\mathrm{mnp}}_{\mathrm{near}}, c^{\mathrm{env}}_{\mathrm{near}}}$ 
   
        \EndFor
        \EndWhile
        
        \State \Return $\mathcal{T}$
    
    \Function{Extend}{$m$(mode), $q_{\mathrm{near}}$, $q_{\mathrm{rand}}$, $c^{\mathrm{mnp}}_{\mathrm{near}}$, $c^{\mathrm{env}}_{\mathrm{near}}$}
    \State $v \gets \Call{closest-feasible-vel}{q_{\mathrm{near}}, q_{\mathrm{rand}}, c^{\mathrm{mnp}}_{\mathrm{near}}, c^{\mathrm{env}}_{\mathrm{near}},m}$
    \If{$\|v\| == 0$}
    \State $c^{\mathrm{mnp}}_{\mathrm{new}} \gets \Call{change-manip-contact}{q_{\mathrm{near}}, c^{\mathrm{mnp}}_{\mathrm{near}}, m}$
    \Else 
    \State $c^{\mathrm{mnp}}_{\mathrm{new}} \gets c^{\mathrm{mnp}}_{\mathrm{near}}$
    \EndIf
    \State $q_{\mathrm{new}} \gets \Call{forward-integrate}{q_{\mathrm{near}}, q_{\mathrm{rand}}, c^{\mathrm{mnp}}_{\mathrm{new}}, m}$
     \If {$q_{\mathrm{new}} \neq q_{\mathrm{near}}$} 
    \State $\Call{assign-finger-contact}{q_{\mathrm{new}},c^{\mathrm{mnp}}_{\mathrm{new}}}$ 
    \State $\mathcal{T}\text{.add-node}(q_{\mathrm{new}})$
    \State $\mathcal{T}\text{.add-edge}(q_{\mathrm{near}},q_{\mathrm{new}})$
    \EndIf
    \State \Return 
    \EndFunction
    \end{algorithmic}
\end{algorithm}

Our planner is presented in Algorithm \ref{alg:rrt}. An object configuration $q_{\mathrm{rand}}$ is drawn from the configuration space through \textsc{sample-object-config} function, where $q_{\mathrm{rand}}$ has user defined possibility $p$ of being a random sample and $1-p$ of being $q_{\mathrm{goal}}$. For $q_{\mathrm{rand}}$, its nearest neighbor $q_{\mathrm{near}}$ in the tree $\mathcal{T}$ is found through a weighted $SE(2)$ metric: 
\begin{equation}
\begin{split}
    d(q_1, q_2) = \sqrt{(q_1^{(x)} - q_2^{(x)})^2 + (q_1^{(y)} - q_2^{(y)})^2} \\+ w_r \min(|q_1^{(\theta)}-q_2^{(\theta)})|,2\pi - |q_1^{(\theta)}-q_2^{(\theta)})|)
\end{split}
\end{equation}
where $w_r$ is the weight that indicates the importance of rotation in the tasks.
Then the properties of node $q_{\mathrm{near}}$, previous finger contacts $c^{\mathrm{mnp}}_{\mathrm{near}}$ and environment contacts $c^{\mathrm{env}}_{\mathrm{near}}$, are obtained from $\mathcal{T}$. 
Function $\textsc{contact-mode-enumeration}$ enumerates all possible environment contact modes $\mathcal{M}$ for $q_{\mathrm{near}}$ (manipulator contact modes are always \textit{fixed}). Under each environment contact mode $m \in \mathcal{M}$, the \textsc{extend} function is performed. The \textsc{extend} function projects the $q_{\mathrm{rand}}$ onto the contact mode manifold through the forward integration process and switch the manipulator contacts if necessary. If \textsc{extend} is successful, a new node $q_{\mathrm{new}}$ and a new edge between $q_{\mathrm{new}}$ and $q_{\mathrm{near}}$ will be added to $\mathcal{T}$. 

In the \textsc{extend} function, we first check the feasibility of the motion towards $q_{\mathrm{rand}}$ through function $\textsc{closest-feasible-vel}$, by solving the quadratic programming problem in Equation \ref{eqn:optvel}. 
If there exists a non-zero solution, the motion is feasible and previous manipulator contacts can remain unchanged. If the motion is not feasible, or the finger contacts collide with the environment, \textsc{change-manip-contact} tries to relocate a random number of manipulator contacts to randomly sampled collision-free contact locations. The relocation is successful if static equilibrium can still be maintained without the relocating finger contacts. The outputs are new manipulator contacts $c^{\mathrm{mnp}}_{\mathrm{new}}$. Finally, $q_{\mathrm{new}}$
is computed by function \textsc{forward-integrate} as described in Section \ref{sec:forward}.  


Besides the main algorithm shown above, we can evaluate the robustness of a motion to force disturbance under the specified contact mode by the stability margin score in\cite{hou2020manipulation}. This stability margin can be used to filter out potentially unstable motions in the planning process --- only a node with enough stability margin will be added to the tree. The plans for real robot experiments in Section \ref{sec:robo_experiments} are generated this way. It can also be used to select better manipulator contacts in \textsc{change-manip-contact}.

\subsection{Probabilistic Completeness}
\label{sec:pc}

Let $\mathcal{Q}$ be the object configuration space, in our case $\mathcal{Q} \subseteq SE(2)$. We denote $\mathcal{B}^d_\delta (q)$ as an open $d$-dimensional ball of radius $\delta$ centered at $q \in \mathbf{R}^n$. Let $\mathcal{U}$ be the parameter space of $N$ finger contacts on the object surface.  


\theoremstyle{definition}
\begin{definition}
\label{def:traj}
A quasistatic trajectory $\pi: [0, t_T] \to \mathcal{Q} \times U$ is with clearance $\delta$ so that: at every point $\pi(t_i)$, there exists an open ball $\mathcal{B}^{d_i+N}_\delta(\pi(t_i))$ on the Cartesian product of its $d_i$-dimensional contact mode manifold $\mathcal{Q}_{m_i}$ and the manipulator contact set $U$, where static equilibrium of the object can be maintained.
\end{definition}

We assume that there exists a trajectory $\pi: [0, t_T] \to \mathcal{Q} \times U$ with clearance $\delta$. The trajectory begins at $q_{\mathrm{start}}$ and terminates at $q_{\mathrm{goal}}$. The trajectory also captures discrete changes including adjacent contact mode transitions, contact encounters, and finger contact switches: there is a point on the trajectory at every discrete change. 
We prove that if such trajectory exists, a random tree framework shown in Algorithm \ref{alg:rrt-complete} is probabilistic complete of finding it.






\begin{lemma}
\label{lemma:nonzero}
Suppose that RRT has reached $\pi'(t_i) \in \mathcal{B}^{d_i+N}_\delta(\pi(t_i))$, $\pi'(t_{i+1}) \in \mathcal{B}^{d_{i+1}+N}_\delta(\pi(t_{i+1}))$ under contact mode $m_{i+1}$ has nonzero probability of being sampled.
\end{lemma}

\noindent\textit{Proof.}(Sketch)
There are three possible situations for $\pi(t_{i+1})$, we prove them all have nonzero sampling probability:
(1)The contact state in $\pi(t_{i+1})$ doesn't change. The set of contact modes $\mathcal{M}_i = \mathcal{M}_{i+1}$. There is no manipulator contact switch. 
(2)There is a manipulator contact switch $u_i, u_{i+1} \in U$, $u_{i} \neq u_{i+1}$.
(3)Only the contact state is changed $\mathcal{M}_i \neq \mathcal{M}_{i+1}$ (encounter new environment contacts). 

Situation (1): As Algorithm \ref{alg:rrt-complete} iterate every possible contact mode for every existing node, contact mode $m_{i+1}$ will be visited with possibility 1. A random sample $x_{\mathrm{rand}}$ is drawn and project onto the manifold for $m_{i+1}$ through forward integration in \ref{sec:forward}. By Theorem 5 in \cite{berenson2011constrained}, the projection sampling using a project operator with Property \ref{lemma:project} covers the manifold, which means the probability to place a sample into $\mathcal{B}^{d_{i+1}+N}_\delta(\pi(t_{i+1}))$ is nonzero. 

Situation (2): Since static equilibrium will be maintained within clearance $\delta$. The probability of the manipulator contacts to be resampled from 
to $B^N_\delta(u_{i+1})$ is nonzero. 

Situation (3):
New environment contacts are encountered at the boundaries of the contact mode manifold. As described in Section \ref{sec:forward}, linear interpolation is perform to go back to the boundary when penetration happens. The linear interpolation is also a projection operator with Property \ref{lemma:project}. 
By Theorem 5 in \cite{berenson2011constrained}, the probability to place a sample into $\mathcal{B}^{d_{i+1}+N}_\delta(\pi(t_{i+1}))$ is nonzero. 



\begin{theorem}
If there exist a trajectory satisfying Definition \ref{def:traj}, the probability of Algorithm \ref{alg:rrt-complete} to find a trajectory within its clearance is nonzero.
\end{theorem}

\begin{proof}(Sketch)
At $t_0$, we have $x_0 = x_{\mathrm{start}}$ and the initial manipulator contacts $u_0$ are randomly sampled in $\mathcal{U}$ (line 8, Algorithm \ref{alg:rrt-complete}). The probability of the initial sample to be within $\mathcal{B}^{3+N}_\delta(\pi(t_0))$ is nonzero ($|B^2_\delta|/|\mathcal{U}|$).
By Lemma \ref{lemma:nonzero}, there is nonzero probability of reach from $\mathcal{B}^{d_i+N}_\delta(\pi(t_i))$ to $\mathcal{B}^{d_{i+1}+N}_\delta(\pi(t_{i+1}))$. By proof of induction, the probability of Algorithm \ref{alg:rrt-complete} finding a solution is nonzero.
\end{proof}
\begin{figure}
\vspace{-1.5mm}
\end{figure}
\begin{algorithm}[t]
    \small
    \caption{Contact Mode Based Manipulation Random Tree Planner (probabilistic complete)} \label{alg:rrt-complete}
    \algorithmicrequire $q_{\mathrm{start}}$, $q_{\mathrm{goal}}$ \newline
    \algorithmicensure tree $\mathcal{T}$
    \begin{algorithmic}[1]
        \State $\mathcal{T}\text{.add-node}(q_{\mathrm{start}})$ 
        \While {(\textit{Goal region has not been reached})}
        \State $q_{r} \gets \Call{sample-object-config}{q_{\mathrm{goal}}}$
        \For{$q_n \in \mathcal{T}$}
        \Comment{Iterate every node}
        \State $c^{\mathrm{mnp}}_n \gets \Call{node-property}{q_n}$
        \State $c^{\mathrm{env}}_n \gets \Call{collision-detection}{q_n}$
        \State $\mathcal{M} \gets \Call{contact-mode-enumeration}{q_n}$
        \For {$m \in \mathcal{M}$} 
        \Comment{Iterate every contact mode}
        \State $c^{\mathrm{mnp}}_{\mathrm{new}} \gets \Call{change-manip-contact}{q_{n}, c^{\mathrm{mnp}}_{n}, m}$
        \State $q_{\mathrm{new}} \gets \Call{forward-integrate}{q_{n}, q_{r}, c^{\mathrm{mnp}}_{\mathrm{new}}, m}$
        \If {$q_{\mathrm{new}} \neq q_{\mathrm{near}}$} 
        \State $\Call{assign-finger-contact}{q_{\mathrm{new}},c^{\mathrm{mnp}}_{\mathrm{new}}}$ 
        \State $\mathcal{T}\text{.add-node}(q_{\mathrm{new}})$
        \State $\mathcal{T}\text{.add-edge}{(q_n, q_{\mathrm{new}})}$
        \EndIf
        \EndFor
        \EndFor
        \EndWhile
        \State \Return $\mathcal{T}$
    \end{algorithmic}
\end{algorithm}

\section{Experiments and Results}
\label{sec:experiments}





\subsection{Experiment Validation}

\begin{figure*}
    \centering
    \vspace{0.1cm}
    \includegraphics[width=\textwidth]{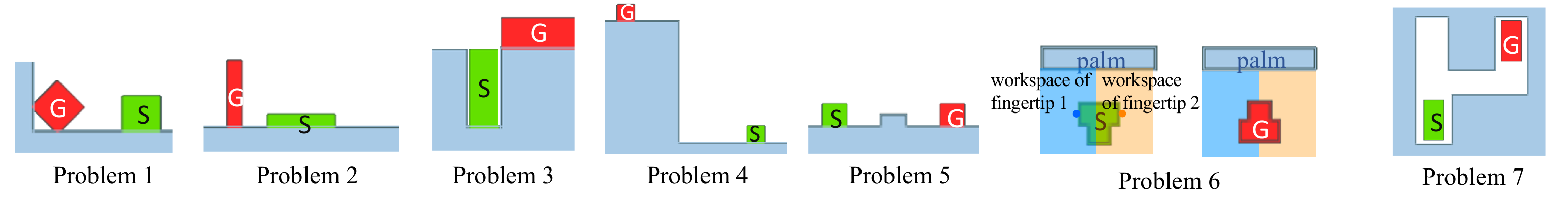}
    \caption{The start configurations (green, ``S'') and the goal configurations (red, ``G'') for each test problems. Problem 1-6 are in the 2D gravity plane. Problem 7 is a planar manipulation problem.}
    \label{fig:goal}
    \vspace{-0.4cm}
\end{figure*}

\begin{figure*}
    \centering
    \includegraphics[width=\textwidth]{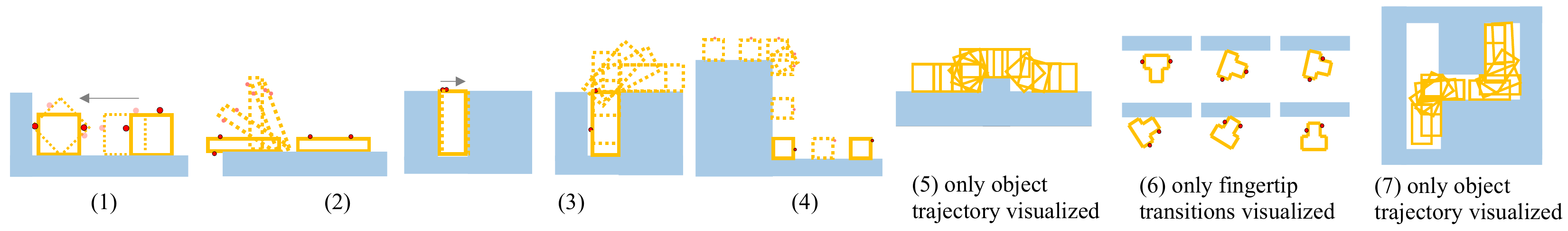}
    \caption{Representative solutions generated by our planner. Object configurations drawn in solid lines for Problem 1-4 indicate that manipulator contact changes are made here, where new manipulator contacts are drawn with red dots. For Problem 5 and 7, as there are many manipulator contact changes, we only visualize the object trajectories. For problem 6, we only visualize where manipulator contact changes happen.}
    \label{fig:solutions}
    
\end{figure*}

We test our planner on several 2D scenarios that need dexterous strategies. 
Figure \ref{fig:goal} gives an overview of all scenarios. All scenarios are designed based on real-world manipulation applications and observations of human manipulation behaviors, to show that our planner has the potentials to address such problems, or automatically plan similar behaviors. Table \ref{tab:experiment} provides the planning statistics for all the problems under 10 random runs performed on a personal computer with Intel Core i7-6560U 2.2GHz CPU.
No parameter has been specifically tuned for any problem above. 
Our results show that this planner is suitable for a variety of 2D manipulation tasks. We visualize one representative solution of each problem in Figure \ref{fig:solutions}. Please check our supplementary video for all solutions for the 10 random runs of each problem.




\begin{table*}[]
    \centering
    \begin{tabular}{c|c|cccccccc}
    \toprule

    \multicolumn{2}{c|}{Problem} & 1 & 2 & 3 & 4 & 5 & 6 & 7  \\  \midrule

    \multicolumn{2}{c|}{Success (Total nodes $<$200)} & 10/10 & 10/10 & 10/10 & 9/10 & 8/10 & 6/10 (nodes$<$1000) & 10/10  \\ \midrule
    \multirow{3}{*}{Time (second)} & Min & 0.11 & 0.64 & 0.45 & 3.18 & 2.05 & 1.98 & 0.71\\
         &  Median & 0.26 & 1.97 & 1.46 & 4.25 & 7.56 & 16.97 & 3.62\\
         &  Max & 0.61 & 12.84 & 3.62 & 8.74 & 15.78 & 43.12 & 5.54\\\midrule
    \multirow{2}{*}{Nodes (median)} & in tree  & 7& 14.5 & 16.5 & 44 & 107 & 480 & 68.5 \\
    & in path & 4 & 5 & 8 & 11 & 17 & 25.5 & 18.5 \\ \midrule
    \multicolumn{2}{c|}{Contact Modes in Path (median)} & 3  & 3 & 5 & 5 & 9 & 4.5 (fingertip states) & 9.5 \\ \bottomrule
    \end{tabular}
    \caption{Planning time, tree sizes and solution sizes of our planner for Problem 1-7. Contact Modes in Path: the number of different contact states and contact modes in the solution path; for Problem 6 we show the number of different finger contact states instead. }
    \label{tab:experiment}
    \vspace{-0.5cm}
\end{table*}


\noindent \textbf{Problem 1: Move and Pivot a Block.} A two-point manipulator needs to move and pivot a heavy block that it cannot directly pick up.
This planner consistently generates the strategies of first pushing and then pivoting the block. As summarized in Table \ref{tab:experiment}, it can get all solutions at around 0.3 seconds with a small number of total nodes explored. 

\noindent \textbf{Problem 2: Pick up a Blade.} The manipulator cannot contact the short edges of a blade-like object. A common strategy used by humans is to slide the object out of the edge of the table to expose its bottom surface for grasping. Our planner is able to consistently generate similar strategies shown in Figure \ref{fig:solutions}(2) at about 2 seconds (Table \ref{tab:experiment})

\noindent \textbf{Problem 3: Unpacking.} The hardest part of picking up well-packed objects is the first object. The small gaps prevent two fingers from picking up the first object directly. 
As shown in Figure \ref{fig:solutions}(3), this method plans to push the object to create a wider gap that lets the fingers slide up the object from one side. Similar strategies have been observed in decluttering actions in human grasping \cite{nakamura2017complexities}.   

\noindent \textbf{Problem 4: Move a Block up to the Cliff.} A single-point manipulator needs to take a block up to a cliff, as shown in Figure \ref{fig:goal}. Only given the start and goal object configuration, our planner successfully generates strategies of push on the floor, push along the cliff, and pivot under gravity. 

\noindent \textbf{Problem 5: Manipulator's Obstacle Course.} This problem shows our planner capability for planning over longer horizons. Due to force limits, the manipulator needs to get the object through a simple 'obstacle course' without picking it up. 
This planner can plan different strategies with many contact switches to get through the obstacles, by using the contacts with the environment.


\noindent \textbf{Problem 6: In-hand Manipulation for Non-Convex Object} This is a simplified in-hand manipulation scenario to show that our method works for non-convex polytopic-shape objects and can generate plans that are not even intuitive to human. As shown in Figure \ref{fig:goal}, the blue rectangle is the palm, and the workspace of each fingertip is a half side of the hand. The goal is to flip the T-shape object 180 degrees with respect to the palm. The planned strategy is to use the palm as extra support to change finger locations. Due to the sampling-based planning nature, the redundant motion problem is obvious here. Trajectory optimization can be used to smoothen the motions as shown in the video. 

\noindent \textbf{Problem 7: Narrow Passage for Planar Pushing} Our method is also general enough for planar manipulation. The problem is to push the object through the narrow hallway with one contact. As shown in the video and Figure \ref{fig:solutions}, our planner consistently generates the strategies of utilizing the environment contacts to move and reorient the object.


\subsection{Robot Experiments}
\label{sec:robo_experiments}
We validate our plans of Problem 4 and 5 on an ABB IRB-120 robot with a point-finger. The plans are generated by Algorithm~\ref{alg:rrt} with a stability margin filter as describe in Section~\ref{sec:main-alg}. The plans are executed by a hybrid force velocity controller\cite{hou2020manipulation}. The experiments are presented in Figure \ref{fig:robot} and in the supplementary video.
\begin{figure}
    \centering
    \includegraphics[height=0.3\columnwidth]{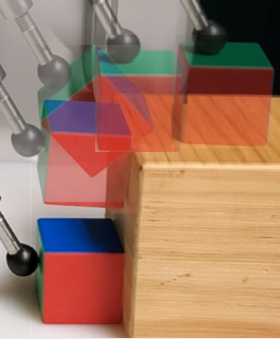}
    \includegraphics[height=0.3\columnwidth]{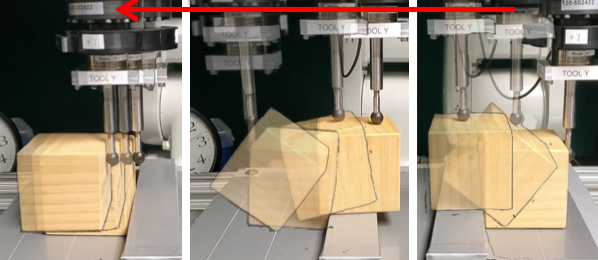}
    \caption{Two trajectories recorded during robot experiments. Left: push a block up to the cliff. Right: take the block through the obstacle using one contact. Three manipulator contact changes are planned.}
    \vspace{-0.5cm}
    \label{fig:robot}
\end{figure}


\section{Conclusion and Discussion}
\label{sec:discussion}

In this paper, we propose a contact mode guided sampling-based 2D manipulation planning framework which generates quasistatic motion trajectory and contact transitions. This method has the potentials to be used in hand manipulation where contact points are fingertips, and in multi-robot manipulation where contact points are mobile robots. 

One limitation of this method is that every motion needs to be quasistatic, which makes our algorithm fails to plan simple dynamic motions like releasing the fingers to drop the object. While quasi-dynamic or dynamic formulations may double the dimensions of the search space, one potential solution is to perform short periods of dynamic simulation between two quasi-static states under some special circumstances. This planner can also be augmented by local planners for fingertip placement and fingertip sliding. Moreover, the contact mode guidance can also be applied on other possibly more efficient constrained sampling-based planning frameworks \cite{kingston2019exploring}.



\bibliographystyle{IEEEtran.bst}
\bibliography{IEEEabrv, references}

\end{document}